%% file: icra2012_techreport.tex
\documentclass[letterpaper, 10 pt, conference]{ieeeconf}  

\IEEEoverridecommandlockouts
\makeatletter

\let\proof\@undefined
\let\endproof\@undefined
\makeatother

\usepackage{makeidx}         
\usepackage{graphicx}        
\usepackage{multicol}        
\usepackage[bottom]{footmisc}
\usepackage{latexsym}

\usepackage[caption=false,font=footnotesize]{subfig}
\usepackage{cite,color,comment,xspace}
\usepackage{amsfonts}
\usepackage{amsmath}
\usepackage{amssymb}
\usepackage{amsthm}
\usepackage{algorithm}
\usepackage{algorithmic}
\usepackage{url,cite}
\usepackage{times}
\usepackage{enumerate}
\usepackage{epstopdf}
\usepackage{epsfig}
\usepackage{array}
\usepackage{fixltx2e}

\graphicspath{{./}}

\newtheorem{theorem}{Theorem}[section]
\newtheorem{proposition}[theorem]{Proposition}

\newtheorem{definition}[theorem]{Definition}

\newtheorem{remark}[theorem]{Remark}

\newtheorem{problem}[theorem]{Problem}

\newcommand{\N}{\mathcal{N}}
\newcommand{\M}{\mathcal{M}}
\newcommand{\R}{\mathcal{R}}
\newcommand{\Prod}{\mathcal{P}}

\newcommand{\bq}{\textbf{q}}

\newcommand{\bo}{\textbf{o}}

\newcommand{\Land}{\wedge}
\newcommand{\Lor}{\vee}
\newcommand{\Next}{\mathsf{X}\, }
\newcommand{\Always}{\mathsf{G} \,}
\newcommand{\Event}{\mathsf{F} \,}
\newcommand{\Until}{\, \mathsf{U} \,}

\newcommand{\n}{\nonumber\\}
\newcommand{\st}{\,|\,}

\newcommand{\mc}[1]{\mathcal{#1}}

\newcommand{\be}{\begin{equation}}
\newcommand{\ee}{\end{equation}}
\newcommand{\ben}{\begin{equation*}}
\newcommand{\een}{\end{equation*}}
\newcommand{\bea}{\begin{eqnarray}}
\newcommand{\eea}{\end{eqnarray}}
\newcommand{\bean}{\begin{eqnarray*}}
\newcommand{\eean}{\end{eqnarray*}}
\newcommand{\ba}{\begin{array}}
\newcommand{\ea}{\end{array}}
\newcommand{\leftm}{\left[\begin{array}}
\newcommand{\rightm}{\end{array}\right]}

\newcommand{\ie}{{\it i.e., }}
\newcommand{\eg}{{\it e.g., }}

\newcommand\oprocendsymbol{\hbox{$\bullet$}}
\newcommand\oprocend{\relax\ifmmode\else\unskip\hfill\fi\oprocendsymbol}

\input{mymacros_2cln}
\input{lpsymbols}

\begin{document}

\title{Temporal Logic Motion Control using Actor-Critic Methods \\ -- Technical Report --~\authorrefmark{1}\thanks{* Research partially supported by the NSF under grant EFRI-0735974, by the DOE under grant DE-FG52-06NA27490, by the ODDR\&E MURI10 program under grant N00014-10-1-0952, and by ONR MURI under grant N00014-09-1051.}}

\author{
Xu Chu Ding\authorrefmark{2}, Jing Wang\authorrefmark{3}, Morteza Lahijanian\authorrefmark{3}, Ioannis Ch. Paschalidis\authorrefmark{3}, and Calin A. Belta\authorrefmark{3}
\thanks{$\dagger$ Xu Chu Ding is with Embedded Systems and Networks group, United Technologies Research Center, East Hartford, CT 06108 ({\tt dingx@utrc.utc.com}).}
\thanks{Jing Wang, Morteza Lahijanian, Ioannis Ch. Paschalidis, and Calin A. Belta
are with the Division of  System Eng., Dept. of Mechanical Eng., Dept. of Electrical \&  
Computer Eng., and Dept. of Mechanical Eng., Boston University, Boston, MA 02215 {\tt (\{wangjing,morteza, yannisp, cbelta\}@bu.edu}), respectively.}
}

\maketitle

\begin{abstract}
In this paper, we consider the problem of deploying a robot from a specification given as a temporal logic statement about some properties satisfied by the regions of a large, partitioned environment. We assume that the robot has noisy sensors and actuators and model its motion through the regions of the environment as a 
Markov Decision Process (MDP). The robot control problem becomes finding the control policy maximizing the probability of satisfying the temporal logic task 
on the MDP. For a large environment, obtaining transition probabilities for each state-action pair, as well as solving the necessary optimization problem for the optimal policy are usually not computationally feasible. To address these issues, we propose an approximate dynamic programming framework based on a least-square temporal difference learning method of the actor-critic type. This framework operates on sample paths of the robot and optimizes a randomized control policy with respect to a small set of parameters. The transition probabilities are obtained only when needed.   Hardware-in-the-loop simulations confirm that convergence of the parameters translates to an approximately optimal policy.  
\end{abstract}


{\allowdisplaybreaks
\section{Introduction}\label{sec:intro}
One major goal in robot motion planning and control is to specify a mission task in an expressive and high-level language and convert the task automatically to a control strategy for the robot.  The robot is subject to mechanical constraints, actuation and measurement noise, and limited communication and sensing capabilities. The challenge in this area is the development of a computationally efficient framework accommodating both the robot constraints and the uncertainty of the environment, while allowing for a large spectrum of task specifications. 

In recent years, temporal logics such as Linear Temporal Logic (LTL) and Computation Tree Logic (CTL) have been promoted as formal task specification languages for robotic applications \cite{Hadas-ICRA07,Karaman_mu_09,Loizou04,Quottrup04,Tok-Ufuk-Murray-CDC09,bhatia2010sampling}.  They are appealing 
due to their high expressivity and closeness to human language. 
Moreover, several existing formal verification 
\cite{Clarke99,baier2008principles} and synthesis \cite{baier2008principles}
tools can be adapted to generate motion plans and provably correct control strategies for the robots.

In this paper, we assume that the robot model in the environment is described by a (finite) Markov Decision Process (MDP).  In this model, the robot can precisely determine its current state, and by applying an action (corresponding to a motion primitive) enabled at each state, it triggers a transition to an adjacent state with a fixed probability. 
We are interested in controlling the MDP robot model such that it maximizes the probability of satisfying a temporal logic formula over a set of properties satisfied at the states of the MDP. 
By adapting existing probabilistic model checking  \cite{baier2008principles,de1997formal,vardi1999probabilistic} and synthesis \cite{courcoubetis1998markov,baier2004controller} algorithms, we recently developed such computational frameworks for formulas of LTL \cite{IFAC2011_LTL} and a fragment of probabilistic CTL \cite{LaWaAnBe-ICRA10}.  

With the above approaches, an optimal control policy can be generated to maximize the satisfying probability, given that the transition probabilities are known for each state-action pair of the MDP, which can be computed by using a Monte-Carlo method and repeated forward simulations. However, it is often not feasible for realistic robotic applications to obtain the transition probabilities for each state-action pair, even if an accurate model or a simulator of the robot in the environment is available. 
Moreover, the problem size is even larger when considering temporal logic specifications. 
For example, in order to find an optimal policy for an MDP satisfying an LTL formula, one need to solve a dynamical programming problem on the product between the original MDP and a Rabin automaton representing the formula.  As such, exact solution can be computationally prohibitive for realistic settings.

In this paper, we show that approximate dynamic programming \cite{si2004handbook} can be effectively used to address the above limitations.  For large dynamic programming problems, an approximately optimal solution can be provided using actor-critic algorithms \cite{barto1983neuronlike}. In particular, actor-critic algorithms with Least Squares Temporal Difference (LSTD) learning have been shown recently to be a powerful tool to solve large-sized problems\cite{paschalidis2009actor, konda2004actor}.  This paper extends from \cite{estanjini2011least}, in which we proposed an actor-critic method for maximal reachability (MRP) problems, \ie maximizing the probability of reaching a set of states, to a computational framework that finds a control policy such that the probability of its paths satisfying an arbitrary LTL formula is locally optimal over a set of parameters.  This set of parameters is designed to tailor to this class of approximate dynamical programming problems.  

Our proposed algorithm produces a \emph{randomized policy}, which gives a probability distribution over enabled actions at a state.  
Our method requires transition probabilities to be generated only along sample paths, and is therefore particularly suitable for robotic applications. To the best of our knowledge, this is the first of combining temporal logic formal synthesis with actor-critic type methods. We illustrate the algorithms with hardware-in-the-loop simulations using an accurate simulator of our Robotic InDoor Environment (RIDE) platform \cite{RIDE-BU}.


\paragraph*{\bf Notation} We use bold letters to denote sequences and vectors. Vectors are assumed to be column vectors.  Transpose of a vector $\bx$ is denoted by $\bx^{\mathtt{T}}$.  $\|\cdot\|$ stands for the Euclidean norm.  $|S|$ denotes the cardinality of a set $S$.  


\section{Problem Formulation and Approach}
\label{sec:prob}
We consider a robot moving in an environment partitioned into regions such as the Robotic Indoor Environment (RIDE) (see Fig. \ref{fig:robotandsmallenv}).  Each region in the environment is associated with a set of observations.  Observations can be {\bf Un} for unsafe regions, or {\bf Up} for a region where the robot can upload data. We assume that the robot can detect its current region.  Moreover, the robot is programmed with a set of motion primitives allowing it to move from a region to an adjacent region.  To capture noise in actuation and sensing, we make the natural assumption that, at a given region, a motion primitive designed to take the robot to a specific adjacent region may take the robot to a different adjacent region.  

\begin{figure}
   \center
   \includegraphics[width=.45\textwidth]{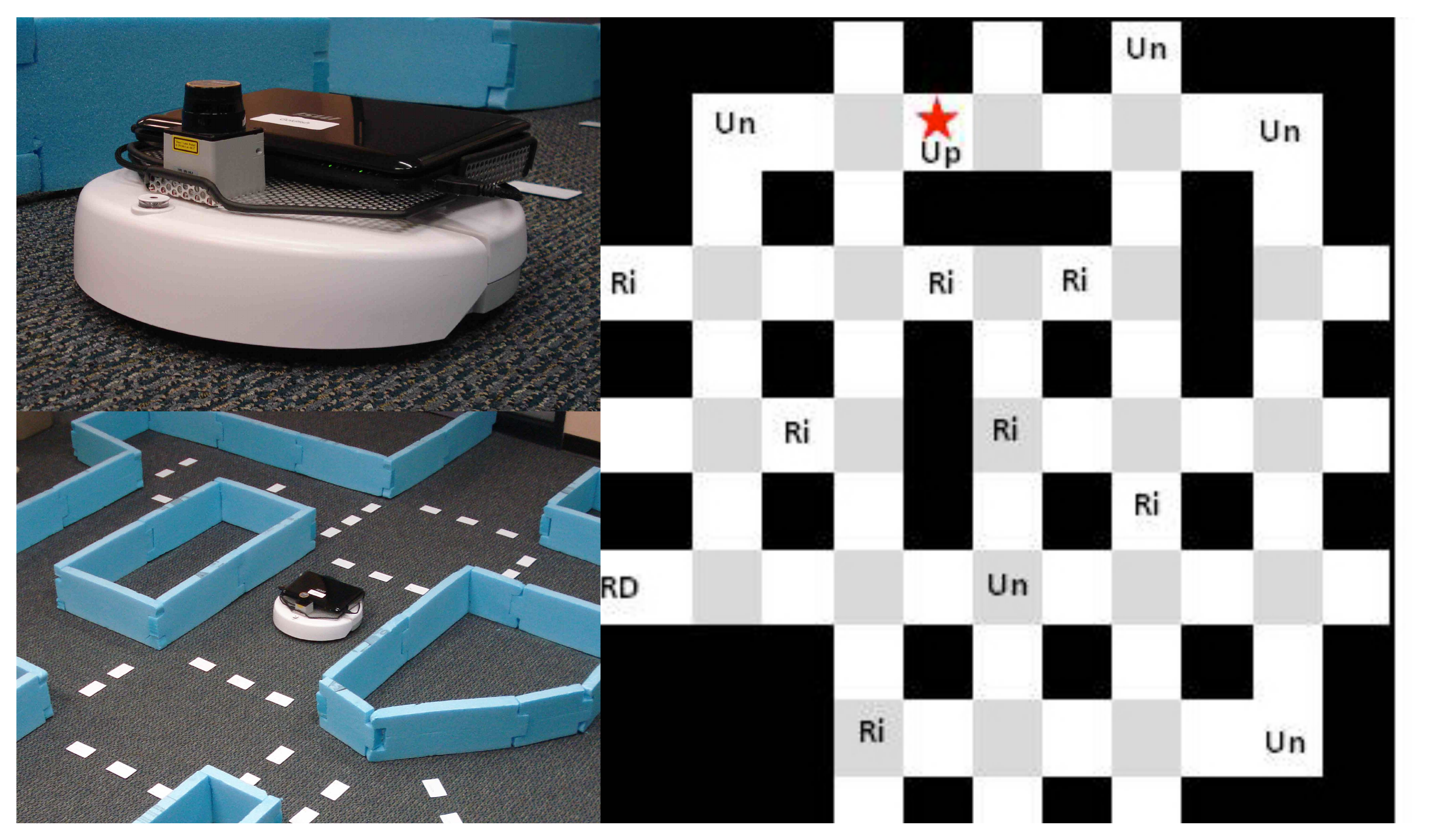} \\
   \caption{Robotic InDoor Environment (RIDE) platform. Left: An iCreate mobile platform moving autonomously through the corridors and intersections of an indoor-like environment. Right: The partial schematics of the environment.   The black blocks represent walls, and the grey and white regions are intersection and corridors, respectively.  The labels inside a region represents observations associated with regions, such as {\bf Un} (unsafe regions) and {\bf Ri} (risky regions).}
   \label{fig:robotandsmallenv}
\end{figure}

Such a robot model naturally leads to a labeled Markov Decision Process (MDP), which is defined below. 
\begin{definition}[Labeled Markov Decision Process]
\label{def:MDP}
A labeled Markov decision process (MDP) is a tuple $\mathcal M=(Q, q_{0}, U, A, P, \Pi, h)$, where 
\begin{enumerate}
 \item $Q=\{1,\ldots,n\}$ is a finite set of states;  
  \item $q_{0}\in Q$ is the initial state; 
 \item $U$ is a finite set of actions;
 \item $A: Q\to U$ maps a state $q\in Q$ to actions enabled at $q$;
 \item $P: Q\times U\times Q\rightarrow [0,1]$ is the transition probability function such that for all $q\in Q$,
$\sum_{q'\in Q}P(q,u,q') = 1$ if $u\in A(q)$, and $P(q,u,q')=0$ for all $q'\in Q$ if $u\notin A(q)$;
 \item $\Pi$ is a set of observations; 
 \item $h: Q\rightarrow 2^{\Pi}$ is the observation map.
\end{enumerate}
\end{definition}

Each state of the MDP $\M$ modeling the robot in the environment corresponds to an ordered set of regions in the environment, while the actions label the motion primitives that can be applied at a region.  For example, a state of $\M$ may be labelled as $I_{1}$-$C_{1}$, which means that the robot is currently at region $C_{1}$, coming from region $I_{1}$.  Each ordered set of regions corresponds to a recent history of the robot trajectory, and is needed to ensure the Markov property (more details on such MDP abstraction of the robot in the environment can be found in \eg \cite{LaWaAnBe-ICRA10}).  The transition probability function $P$ can be obtained through extensive simulations of the robot in the environment.  We assume that there exists an accurate simulator that is capable of generating (computing) the transition probability $P(q,u,\cdot)$ for each state-action pair $q\in Q$ and $u\in A(q)$.  More details of the construction of the MDP model for a robot in the RIDE platform are included in Sec. \ref{sec:Hardware}.  

If the exact transition probabilities are not known,  $\M$ can be seen as a labeled non-deterministic transition system (NTS) $\mathcal M^{\N}=(Q, q_{0}, U, A, P^{\N}, \Pi, h)$, where $P$ in $\M$ is replaced by $P^{\N}:Q\times U\times Q\to\{0,1\}$, and $P^{\N}(q,u,q')=1$ indicates a possible transition from $q$ to $q'$ applying an enabled action $u\in A(q)$; if $P^\mathcal{N}(q,u,q')=0$, then the transition from $q$ to $q'$ is not possible under $u$.


A path on $\M$ is a sequence of states $\bq=q_{0}q_{1}\ldots$ such that for all $k\geq 0$, there exists $u_{k}\in A(q_{k})$ such that $P(q_{k},u_{k},q_{k+1})>0$.  Along a path $\bq=q_{0}q_{1}\ldots$, $q_{k}$ is said to be the state at time $k$. The trajectory of the robot in the environment is represented by a path $\bq$ on $\M$ (which corresponds to a sequence of regions in the environment).  A path $\bq=q_{1}q_{2}\ldots$ generates a sequence of observations $\bh(\bq):=o_{1}o_{2}\ldots$, where $o_{k}=h(q_{k})$ for all $k\geq 0$.   We call $\bo=\bh(\bq)$ the word generated by $\bq$.  

\begin{definition}[Policy]
\label{def:policy}
A control policy for an MDP $\M$ is an infinite sequence $M=\mu_{0}\mu_{1}\ldots$, where $\mu_{k}:Q\times U\to[0,1]$ is such that $\sum_{u\in A(q)}\mu_{k}(q,u)=1$, for all $k\geq 0$.  
\end{definition}
Namely, at time $k$, $\mu_{k}(q,\cdot)$ is a discrete probability distribution over $A(q)$.  If $\mu=\mu_{k}$ for all $k\geq 0$, then $M=\mu\mu\ldots$ is called a \emph{stationary} policy.  If for all $k\geq 0$, $\mu_{k}(q,u)=1$ for some $u$, then $M$ is \emph{deterministic}; otherwise, $M$ is \emph{randomized}.  Given a policy $M$, we can then generate a set of paths on $\M$, by applying $u_{k}$ with probability $\mu_{k}(q_{k},u_{k})$ at state $q_{k}$ for all time $k$.  


We require the trajectory of the robot in the environment to satisfy a rich task specification given as a Linear Temporal Logic (LTL) (see, \eg \cite{baier2008principles,Clarke99}) formula over a set of observations $\Pi$.  An LTL formula over $\Pi$ is evaluated over an (infinite) sequence $\textbf{o}=o_{0}o_{1}\ldots$ (\eg a word generated by a path on $\M$), where $o_{k}\subseteq \Pi$ for all $k\geq 0$.   We denote $\bo\vDash \phi$ if word $\bo$ satisfies the LTL formula $\phi$, and we say $\bq$ satisfies $\phi$ if $\bh(\bq)\vDash \phi$. Roughly, $\phi$ can be constructed from a set of observations $\Pi$, Boolean operators $\neg$ (negation), $\Lor$ (disjunction), $\Land$ (conjunction), $\longrightarrow$ (implication), and temporal operators $\Next$ (next), $\Until$ (until), $\Event$ (eventually), $\Always$ (always).   A variety of robotic tasks can be easily translated to LTL formulas.  For example, the following complex task command in natural language: \emph{``Gather data at locations {\bf Da} infinitely often.  Only reach a risky region {\bf Ri} if valuable data at {\bf VD} can be gathered, and always avoid unsafe regions ({\bf Un})''} can be translated to the LTL formula: 
$$\phi:=\Always \Event \textbf{Da} \Land \Always (\textbf{Ri}\longrightarrow \textbf{VD}) \Land \Always \neg \textbf{Un}.$$


In this paper, we consider the following problem.
\begin{problem}
 \label{prob:solvedprob}
 Given a labeled MDP $\mathcal M=(Q, q_{0}, U, A, P, \Pi, h)$ modeling the motion of a robot in a partitioned environment and a mission task specified as an LTL formula $\phi$ over $\Pi$, find a control policy that maximizes the probability of its path satisfying $\phi$. 
\end{problem}
The probability that paths generated under a policy $M$ satisfy an LTL formula $\phi$ is well defined with a suitable measure over the set of all paths generated by $M$ \cite{baier2008principles}.  

In \cite{IFAC2011_LTL}, we proposed a computational framework to solve Prob. \ref{prob:solvedprob}, by adapting methods from the area of probabilistic model checking \cite{de1997formal,baier2008principles,vardi1999probabilistic}.  However, this framework relies upon the fact that the transition probabilities are known for all state-action pairs.   These transition probabilities are typically not available for robotic applications and computationally expensive to compute.  Moreover, even if the transition probabilities are obtained for each state-action pair, this method still requires solving a linear program on the product of the MDP and the automata representing the formula, which can be very large (thousands or even millions of states).  In this case an approximate method might be more desirable. For these reasons, we instead focus on the following problem.

\begin{problem}
 \label{prob:mainprob}
 Given a labeled NTS $\mathcal M^{\N}=(Q, q_{0}, U, A, P^{\N}, \Pi, h)$ modeling a robot in a partitioned environment, a mission task specified as an LTL formula $\phi$ over $\Pi$, and an accurate simulator to compute transition probabilities $P(q,u,\cdot)$ given a state-action pair $(q,u)$, find a control policy that approximately maximizes the probability of its path satisfying $\phi$. 
\end{problem}

In many robotic applications, the NTS model $\mathcal M^{\N}=(Q, q_{0}, U, A, P^{\N}, \Pi, h)$ can be quickly constructed for the robot in the environment.   Our approach to Prob. \ref{prob:mainprob} can be summarized as follows: First, we proceed to translate the problem to a maximal reachability probability (MRP) problem using $\M^{\N}$ and $\phi$ (Sec. \ref{sec:sec:genMRP}).  We then use an actor critic framework to find a randomized policy giving an approximate solution to the MRP problem (Sec. \ref{sec:sec:A-C}).  The randomized policy is constructed to be a function of a small set of parameters and we find a policy that is locally optimal with respect to these parameters.  The construction of a class of policies suitable for MRP problems without using the transition probabilities is explained in Sec. \ref{sec:sec:RSPdesign}.  The algorithmic framework presented in this paper is summarized in Sec. \ref{sec:sec:overallAlg}.

\section{Control Synthesis}
\label{sec:control}

\subsection{Formulation of the MRP Problem}
\label{sec:sec:genMRP}
The formulation of the MRP problem is based on \cite{IFAC2011_LTL,de1997formal,baier2008principles,vardi1999probabilistic} with modification if needed when using the NTS $\M_{\N}$ instead of $\M$.   We start by converting the LTL formula $\phi$ over $\Pi$ to a so-called deterministic \emph{Rabin automaton}, which is defined as follows.
\begin{definition}[Deterministic Rabin Automaton]
\label{def:DRA}
A deterministic Rabin automaton (DRA) is a tuple $\mathcal R=(S,s_{0},\Sigma,\delta,F)$, where
\begin{enumerate}
 \item $S$ is a finite set of states;
  \item $s_{0}\in S$ is the initial state;
 \item $\Sigma$ is a set of inputs (alphabet);
 \item $\delta:S\times\Sigma\rightarrow S$ is the transition function;
 \item $F=\{(L(1),K(1)),\dots,(L(M),K(M))\}$ is a set of pairs of sets of states such that $L(i),K(i)\subseteq S$ for all $i=1,\dots,M$.
\end{enumerate}
\end{definition}
A run of a Rabin automaton $\mathcal R$, denoted by $\br=s_{0}s_{1}\ldots$, is an infinite sequence of states in $\mathcal R$ such that for each $k\geq 0$, $s_{k+1}\in\delta(s_{k},\alpha)$ for some $\alpha \in \Sigma$.   A run $\br$ is {\it accepting} if there exists a pair $(L,K)\in F$ such that $\br$ intersects with $L$ finitely many times and $K$ infinitely many times.  For any LTL formula $\phi$ over $\Pi$, one can construct a DRA (for which we denote by $\R_{\phi}$) with input alphabet $\Sigma= 2^{\Pi}$ accepting all and only words over $\Pi$ that satisfy $\phi$ (see \cite{gradel2002automata}).


We then obtain an MDP as the product of a labeled MDP $\mathcal M$ and a DRA $\mathcal R_{\phi}$, which captures all paths of $\mathcal M$ satisfying $\phi$.  
Note that this product MDP can only be constructed from an MDP and a deterministic automaton, this is why we require a DRA instead of, \eg a (generally non-deterministic) B\"uchi automaton (see \cite{baier2008principles}).
\begin{definition}[Product MDP]
\label{def:product}
The product MDP $\mathcal M \times \mathcal R_{\phi}$ between a labeled MDP $\mathcal M=(Q, q_{0}, U, A, P, \Pi, h)$ and a DRA $\mathcal R_{\phi}=(S,s_{0},2^{\Pi},\delta,F)$ is an MDP $\mathcal P=(S_{\mathcal P}, s_{\mathcal P0}, U_{\mathcal P}, A_{\mathcal P}, P_{\mathcal P}, \Pi, h_{\mc P})$, where 
\begin{enumerate}
 \item $S_{\mathcal P}= Q\times S$ is a set of states;
 \item $s_{\mathcal P0}=(q_{0},s_{0})$ is the initial state; 
 \item $U_{\mathcal P}= U$ is a set of actions inherited from $\mathcal M$;
 \item $A_{\mathcal P}$ is also inherited from $\M$ and $A_{\mc P}((q,s)):=A(q)$; 
 \item $P_{\mathcal P}$ gives the transition probabilities: 
\ben
P_{\mathcal P}((q,s),u,(q',s'))\! =\!\begin{cases} P(q,u,q') & \textrm{if } q'=\delta(s,h(q)) \\ 0 & \rm otherwise; \end{cases}
\een
\end{enumerate}
\end{definition}
Note that $h_{\mc P}$ is not used in the product MDP.  Moreover, $\mc P$ is associated with pairs of accepting states (similar to a DRA)  $F_{\mathcal P}:=\{(L_{\mathcal P}(1), K_{\mathcal P}(1)),\ldots,(L_{\mathcal P}(M), K_{\mathcal P}(M))\}$ where $L_{\mathcal P}(i)=Q\times L(i)$, $K_{\mathcal P}(i)=Q\times K(i)$, for $i=1,\ldots,M$; 

The product MDP is constructed in a ways such that, given a path $(s_{0},q_{0})(s_{1},q_{1})\ldots$, the corresponding path $s_{0}s_{1}\ldots$ on $\mathcal M$ satisfies $\phi$ if and only if there exists a pair $(L_{\mathcal P}, K_{\mathcal P})\in F_{\mathcal P}$ satisfying the Rabin acceptance condition, \ie the set $K_{\Prod}$ is visited infinitely often and the set $L_{\Prod}$ is visited finitely often.


We can make a very similar product between a labeled NTS $\M^{\N}=(Q, q_{0}, U, A, P^{\N}, \Pi, h)$ and $\R_{\phi}$.  This product is also an NTS, which we denote by $\Prod^{\N}=(S_{\mathcal P}, s_{\mathcal P0}, U_{\mathcal P}, A_{\mathcal P}, P^{\N}_{\mathcal P}, \Pi, h_{\mc P}):=\M^{\N}\times \R_{\phi}$, associated with accepting sets $F_{\mc P}$.   The definition (and the accepting condition) of $\Prod^{\N}$ is exactly the same as for the product MDP.   The only difference between $\Prod^{\N}$ and $\Prod$ is in $P^{\N}_{\mathcal P}$, which is either $0$ or $1$ for every state-action-state tuple. 

From the product $\Prod$ or equivalently $\Prod^{\N}$, we can proceed to construct the MRP problem.  To do so, it is necessary to produce the so-called \emph{accepting maximum end components} (AMECs).  An end component is a subset of an MDP (consisting of a subset of states and a subset of enabled actions at each state) such that for each pair of states $(i,j)$ in $\Prod$, there is a sequence of actions such that $i$ can be reached from $j$ with positive probability, and states outside the component cannot be reached.  
An AMEC of $\Prod$ is the largest end component containing at least one state in $K_{\Prod}$ and no state in $L_{\Prod}$, for a pair $(K_{\Prod}, L_{\Prod})\in F_{\Prod}$.

A procedure to obtain all AMECs of an MDP is outlined in \cite{baier2008principles}.   This procedure is intended to be used for the product MDP $\Prod$, but it can be used without modification to find all AMECs associated with $\Prod$ when $\Prod^{\N}$ is used instead of $\Prod$.  This is because the information needed to construct the AMECs is the set of all possible state transitions at each state, and this information is already contained in $\Prod^{\N}$. 

If we denote $S_{\Prod}^{\star}$ as the union of all states in all AMECs associated with $\Prod$, it has been shown in probabilistic model checking (see \eg \cite{baier2008principles}) that the probability of satisfying the LTL formula is given by the maximal probability of reaching the set $S_{\Prod}^{\star}$ from the initial state $S_{\Prod0}$.   The desired optimal policy can then be obtained as the policy maximizing this probability.   If transition probabilities are available for each state-action pair, then the solution to this MRP problem can be solved as by a linear program (see \cite{puterman1994markov,baier2008principles}).  The resultant optimal policy is deterministic and (\ie $M=\mu\mu\ldots$) on the product MDP $\Prod$.  To implement this policy on $\M$, it is necessary to use the DRA as a feedback automaton to keep track of the current state $s_{\Prod}$ on $\Prod$, and apply the action $u$ where $\mu(s_{\Prod},u)=1$ (since $\mu$ is deterministic).

\begin{remark}
 It is only necessary to find the optimal policy for states not in the set $S_{\Prod}^{\star}$.   This is because by construction, there exists a policy inside any AMEC that almost surely satisfies the LTL formula $\phi$ by reaching a state in $K_{\Prod}$ infinitely often.  This policy can be obtained by simply choosing an action (among the subset of actions retained by the AMEC) at each state randomly, \ie a trivial randomized stationary policy exists that almost surely satisfies $\phi$.
\end{remark}

\subsection{LSTD Actor-Critic Method}
\label{sec:sec:A-C}
We now describe how relevant results in \cite{estanjini2011least} can be applied to solve Prob. \ref{prob:mainprob}.  An approximate dynamic programming algorithm of the actor-critic type was presented in \cite{estanjini2011least}, which obtains a stationary randomized policy (RSP) (see Def. \ref{def:policy}) $M=\mu_{\btheta}\mu_{\btheta}\ldots$, where $\mu_{\btheta}(q,u)$ is a function of the state-action pair $(q,u)$ and $\btheta\in\mathbb R^{n}$, which is a vector of parameters.  For the convenience of notations, we denote an RSP $\mu_{\btheta}\mu_{\btheta}\ldots$ simply by $\mu_{\btheta}$.   In this sub-section we assume that the RSP $\mu_{\btheta}(q,u)$ to be given, and we will describe in Sec. \ref{sec:sec:RSPdesign} on how to design a suitable RSP.

Given an RSP $\mu_{\btheta}$, actor-Critic algorithms can be applied to optimize the parameter vector $\btheta$ by policy gradient estimations. The basic idea is to use stochastic learning techniques to find $\btheta$ that locally optimizes a cost function.  In particular, the algorithm presented in \cite{estanjini2011least} is targeted at Stochastic Shortest Path (SSP) problems commonly studied in literature (see \eg \cite{puterman1994markov}).  Given an MDP $\mathcal M=(Q, q_{0},U, A, P,  \Pi, h)$, a termination state $q^{\star}\in Q$ and a function $g(q,u)$ defining the one-step cost of applying action $u$ at state $q$, the {\em expected total cost} is defined as:
\begin{equation}
\label{eq:SSPcost}
	\bar{\alpha}(\btheta) = \lim_{N\to\infty} E\left\{\sum_{k=0}^{N-1} g(q_k,u_k)\right\},
\end{equation}
where $(q_{k}, u_{k})$ is the state-action pair at time $k$ along a path under RSP $\mu_{\btheta}$.

The SSP problem is formulated as the problem of finding $\btheta^{\star}$ minimizing \eqref{eq:SSPcost}. 
Note that, in general, we assume $q^{\star}$ to be cost-free and absorbing, \ie $g(q^{\star},u)=0$ and $P(q^{\star}, u,q^{\star})=1$ for all $u\in A(q^{\star})$.  Under these conditions, the expected total cost \eqref{eq:SSPcost} is finite. 

We note that an MRP problem as described in Sec. \ref{sec:sec:genMRP} can be immediately converted to an SSP problem.    

\begin{definition}[Conversion from MRP to SSP]
\label{def:conversion}
Given the product MDP $\Prod=(S_{\mathcal P},s_{\mathcal P0}, U_{\mathcal P}, A_{\mathcal P}, P_{\mathcal P}, F_{\mathcal P})$ and a set of states $S^{\star}_{\Prod}\subseteq S_{\Prod}$, the problem of maximizing the probability of reaching $S^{\star}_{\Prod}$ can be converted to an SSP problem by defining a new MDP $\widetilde \Prod=(\widetilde S_{\mathcal P},\tilde s_{\mathcal P0}, \widetilde U_{\mathcal P}, \widetilde A_{\mathcal P}, \widetilde P_{\mathcal P},g_{\Prod})$, where
\begin{enumerate}
 \item $\widetilde S_{\Prod}=(S_{\Prod}\setminus S^{\star}_{\Prod}) \cup \{s_{\Prod}^{\star}\}$, where $s_{\Prod}^{\star}$ is a ``dummy'' terminal state;
 \item $\tilde s_{\mathcal P0}=s_{\mathcal P0}$ (without the loss of generality, we exclude the trivial case where $s_{\mathcal P0}\in S^{\star}_{\Prod}$);
 \item $\widetilde U_{\Prod}=U_{\Prod}$;
 \item $\widetilde A_{\mathcal P}(s_{\Prod})=A_{\mathcal P}(s_{\Prod})$ for all $s_{\Prod}\in S_{\Prod}$, and for the dummy state we set $\widetilde A_{\mathcal P}(s_{\Prod}^{\star})=\widetilde U_{\mathcal P}$;
 \item The transition probability is redefined as follows.  We first define $\bar S_{\Prod}^{\star}$ as the set of states on $\Prod$ that cannot reach $S_{\Prod}^{\star}$ under any policy.  We then define:
\bean
\label{E:tranprob}
&&\widetilde P_{\Prod}(s_{\Prod},u,s_{\Prod}')\nonumber\\&=&\left\{\begin{array}{ll}
\displaystyle\sum_{s''_{\Prod}\in S^{\star}_{\Prod}} P_{\Prod}(s_{\Prod}, u, s''_{\Prod}), &\text{if } s_{\Prod}'=s_{\Prod}^{\star}\\
P_{\Prod}(s_{\Prod}, u, s'_{\Prod}), &\text{if } s'_{\Prod}\in S_{\Prod}\setminus S^{\star}_{\Prod}\end{array}\right.
\eean
for all $s_{\Prod}\in S_{\Prod}\setminus (S^{\star}_{\Prod} \cup \bar S_{\Prod}^{\star})$ and $u\in \widetilde U_{\Prod}$.  Moreover, for all $s_{\Prod}\in \bar S_{\Prod}^{\star}$ and $u\in \widetilde U_{\Prod}$, we set $\widetilde P_{\Prod}(s^{\star}_{\Prod},u,s^{\star}_{\Prod})=1$ and $\widetilde P_{\Prod}(s_{\Prod},u,s_{\Prod0})=1$;
\item For all $s_{\Prod}\in \widetilde S_{\Prod}$ and $u\in \widetilde U_{\Prod}$, we define the one-step cost $g_{\Prod}(s_{\Prod},u)=1$ if $s_{\Prod}\in \bar S^{\star}_{\Prod}$, and $g(s_{\Prod},u)=0$ otherwise. 
\end{enumerate} 
\end{definition}

We have shown in  \cite{estanjini2011least} that the policy minimizing \eqref{eq:SSPcost} for the SSP problem with MDP $\widetilde \Prod$ and the termination state $s_{\Prod}^{\star}$ is a policy maximizing the probability of reaching the set $S^{\star}_{\Prod}$ on $\Prod$, \ie a solution to the MRP problem formulated in Sec. \ref{sec:sec:genMRP}.  

The SSP problem can also be constructed from the NTS $\Prod^{\N}$.  In this case we obtain an NTS $\widetilde \Prod^{\N}(\widetilde S_{\mathcal P},\tilde s_{\mathcal P0}, \widetilde U_{\mathcal P}, \widetilde A_{\mathcal P}, \widetilde P^{\N}_{\mathcal P},g_{\Prod})$, using the exact same construction as Def. \ref{def:conversion}, except for the definition of $\widetilde P^{\N}_{\mathcal P}$.  The transition function $\widetilde P^{\N}_{\mathcal P}(s_{\Prod},u,s_{\Prod}')$ is instead defined as:
\bean
\label{E:tranprob}
&&\widetilde P^{\N}_{\Prod}(s_{\Prod},u,s_{\Prod}')\nonumber\\&=&\left\{\begin{array}{ll}
\underset{s''_{\Prod}\in S^{\star}_{\Prod}}{\max} P^{\N}_{\Prod}(s_{\Prod}, u, s''_{\Prod}), &\text{if } s_{\Prod}'=s_{\Prod}^{\star}\\
P^{\N}_{\Prod}(s_{\Prod}, u, s'_{\Prod}), &\text{if } s'_{\Prod}\in S_{\Prod}\setminus S^{\star}_{\Prod}\end{array}\right.
\eean
for all $s_{\Prod}\in S_{\Prod}\setminus (S^{\star}_{\Prod} \cup \bar S_{\Prod}^{\star})$ and $u\in \widetilde U_{\Prod}$.  Moreover, for all $s_{\Prod}\in \bar S_{\Prod}^{\star}$ and $u\in \widetilde U_{\Prod}$, we set $\widetilde P^{\N}_{\Prod}(s^{\star}_{\Prod},u,s^{\star}_{\Prod})=1$ and $\widetilde P^{\N}_{\Prod}(s_{\Prod},u,s_{\Prod0})=1$.

Once the SSP problem is constructed, the algorithm presented in \cite{estanjini2011least} is an iterative procedure that obtains a policy that locally minimizes the cost function \eqref{eq:SSPcost} by simulating sample paths on $\widetilde \Prod$.  Each sample paths on $\widetilde \Prod$ starts at $s_{\Prod 0}$ and ends when the termination state $s_{\Prod}^{\star}$ is reached.  Since the probabilities is needed only along the sample path, we do not require the MDP $\widetilde \Prod$, but only $\widetilde \Prod^{\N}$.

An actor-critic algorithm operates in the following way: the critic observes state and one-step cost from MDP and uses observed information to update the critic parameters, then the critic parameters are
used to update the policy;  the actor generates the action based on the policy and applies the action to the MDP.  The algorithm stops when the gradient of $\bar\alpha(\btheta)$ is small enough (\ie $\btheta$ is locally optimal). The actor-critic update mechanism is shown in Fig. \ref{actor-critic}. 

\begin{figure}[t]
	\begin{center}
		\includegraphics[scale=.32]{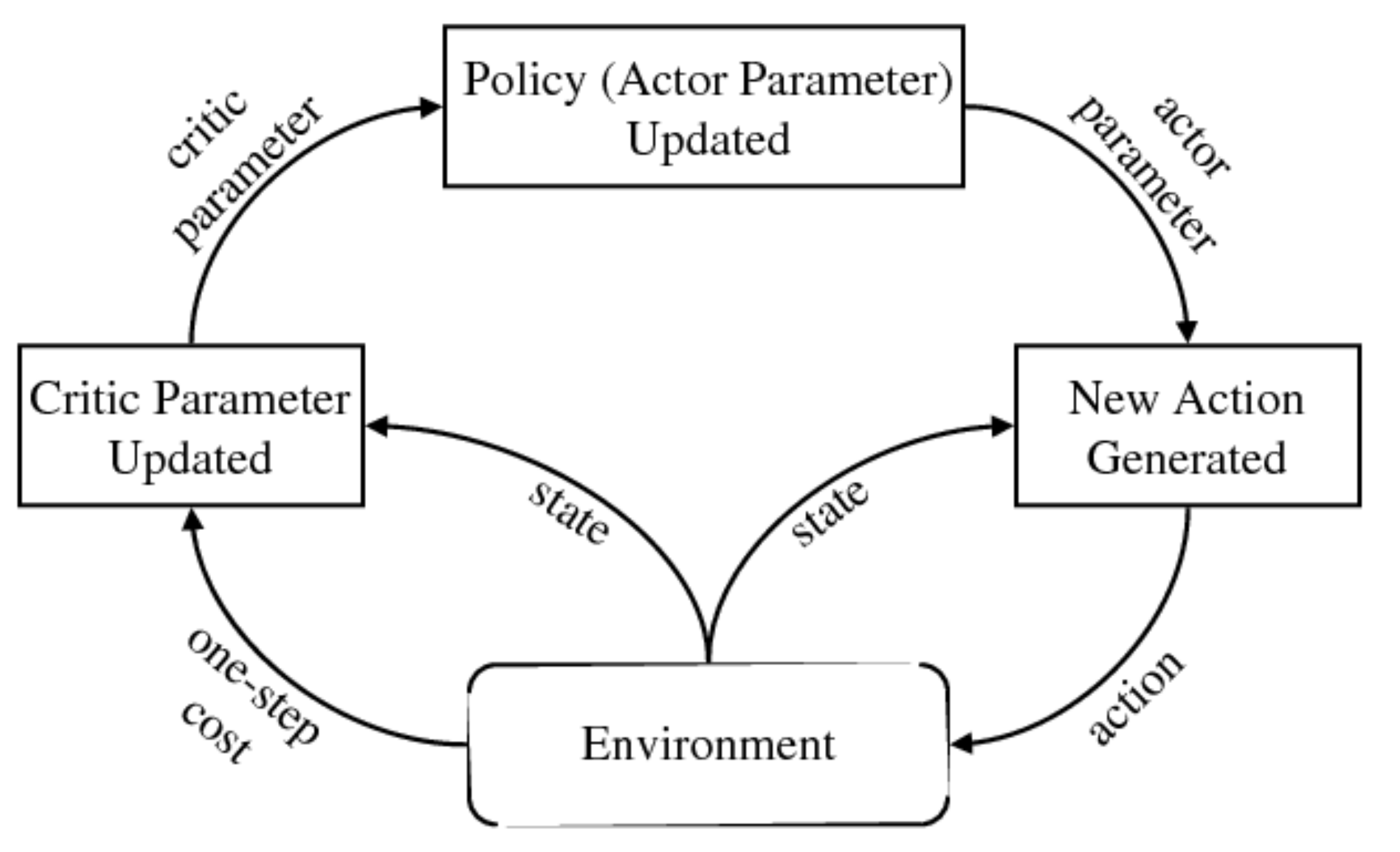}
	\end{center}
	\caption{Diagram illustrating an actor-critic algorithm. }
	\label{actor-critic}
\end{figure}

We summarize the actor-critic update algorithm in Alg. \ref{alg:A-C}, and we note that it does not depend on the form of RSP $\mu_{\btheta}$. The vectors  $\bz_k \in \mathbb{R}^n$,$\bb_k\in \mathbb{R}^n$, $\br_k\in \mathbb{R}^n$ and the matrix $\bA_k\in \mathbb R^{n\times n}$ are updated during each critic update, while simultaneously, the vector $\btheta_{k}\in \mathbb{R}^n$ is updated during each actor update.   Both the critic and actor update depend on 
\be
\label{eq:psitheta}
\bpsi_{\btheta}(x,u):=\nabla_{\btheta} \ln(\mu_{\btheta}(x,u)),
\ee
which is the gradient of the logarithm of $\mu_{\btheta}(x,u)$, to estimate the gradient  $\nabla\bar{\alpha}(\btheta)$.  Lastly, sequence $\{\gamma_k\}$ controls the critic step-size, while $\{\beta_k\}$ and $\Gamma(\br_{k})$ control the actor step-size.  We note that all step-size parameters are positive, and their effect on the convergence rate is discussed in \cite{estanjini2011least}.

The critic update algorithm in Alg. \ref{alg:A-C} is of the LSTD type, which has shown to be superior to other approximate dynamic programming methods in terms of the convergence rate \cite{konda2004actor}.  More detail of this algorithm can be found in \cite{estanjini2011least}.

\begin{algorithm}[ht]
\small
\caption{LSTD Actor-critic algorithm for SSP problems}
\begin{algorithmic}[1]\label{alg:A-C}
\REQUIRE The NTS $\widetilde \Prod^{\N}(\widetilde S_{\mathcal P},\tilde s_{\mathcal P0}, \widetilde U_{\mathcal P}, \widetilde A_{\mathcal P}, \widetilde P^{\N}_{\mathcal P},g_{\Prod})$ with the terminal state $s_{\Prod}^{\star}$, the RSP $\mu_{\btheta}$, and a computation tool to obtain $\widetilde P_{\Prod}(s_{\Prod},u,\cdot)$ for a given $(s_{\Prod},u)$ state-action pair.%
\STATE   \noindent{\bf Initialization:} Set all entries in $\bz_0, \bA_0, \bb_0$ and $\br_0$ to zeros. Let $\btheta_0$ 
take some initial value.  Set initial state $x_{0}:=\tilde s_{\Prod0}$.  Obtain action $u_{0}$ using the RSP $\mu_{\btheta_{0}}$.

\REPEAT 
\STATE Compute the transition probabilities $\widetilde P(x_{k},u_{k},\cdot)$.  
\STATE Obtain the simulated subsequent state $x_{k+1}$ using the transition probabilities $\widetilde P(x_{k},u_{k},\cdot)$.  If $x_{k}= s^{\star}_{\Prod}$,  set $x_{k+1}:=x_{0}$.%
\STATE Obtain action $u_{k+1}$ using the RSP $\mu_{\btheta_{k}}$
\STATE \noindent{\bf Critic Update:}
\bean
 \bz_{k+1} &=& \lambda \bz_k+\bpsi_{\btheta_k}(x_k,u_k)\n
 \bb_{k+1} &=& \bb_k+\displaystyle \gamma_k \left(g(x_k,u_k) \bz_{k}-\bb_k\right)\n
 \bA_{k+1} &=& \bA_k + \displaystyle \gamma_k (\bz_{k}(\bpsi_{\theta_k}^{\mathtt{T}}(x_{k+1},u_{k+1})-\bpsi_{\theta_k}^{\mathtt{T}}(x_k,u_k))\n
 && -\bA_k),\n
 \br_{k+1} &=& -\bA_k^{-1}\bb_k.
\eean
\vskip-.5em
\STATE \noindent{\bf Actor Update:}
\ben
\label{Eq:actor}
	\btheta_{k+1}=\btheta_k-\beta_k\Gamma(\br_k)\br_{k}^{\mathtt{T}}\bpsi_{\theta_k}(x_{k+1},u_{k+1})\bpsi_{\theta_k}(x_{k+1},u_{k+1})
\een
\UNTIL{$||\nabla\bar{\alpha}(\btheta_k)||\leq \epsilon$ for some given $\epsilon$}
\end{algorithmic}
\end{algorithm} 

\subsection{Designing an RSP}
\label{sec:sec:RSPdesign}
In this section we describe a randomized policy suitable to be used in Alg. \ref{alg:A-C} for MRP problems, and do not require the transition probabilities.  We propose a family of RSPs that perform a ``$t$ steps look-ahead''.  This class of policies consider all possible sequences of actions in $t$ steps and obtain a probability for each action sequence.

To simplify notation, for a pair of states $i,j\in \widetilde S_\Prod$, we denote $i\overset{t}{\to} j$ if there is a positive probability of reaching $j$ from $i$ in $t$ step.  This can be quickly verified given $\widetilde P^{\N}_{\Prod}$ without transition probabilities.   At state $i\in \widetilde S_{\Prod}$, we denote an action sequence from $i$ with $t$ steps look-ahead as $e=u_{1} u_{2} \dots u_{t}$, where $u_{k}
\in \widetilde A_{\Prod}(j)$ for some $j$ such that $i\overset{k}{\to} j$, for all $k=1,\dots t$. 
We denote the set of all action sequences from state $i$ as $E(i)$.  Given $e\in E(i)$, we denote $\widetilde P^{\N}_{\Prod}\left(i, e, j \right)=1$ if there is a positive probability of reaching $j$ from $i$ with the action sequence $e$.  This can also be recursively obtained given $\widetilde P^{\N}_{\Prod}\left(i, u, \cdot \right)$.  
 

For each pair of states $i,j\in \widetilde S_\Prod$, we define $d(i,j)$ 
as the minimum number of steps from $i$ to reach $j$ (this again can be obtained quickly from $\widetilde P^{\N}_{\Prod}$ without transition probabilities).  
We denote $j\in N(i)$ if and only if $d(i,j) \leq r_N $, where $r_N$
is a fixed integer given apriori.  If $j\in N(i)$, then we say $i$ is 
in the neighborhood of $j$, and $r_{N}$ represents the radius of the neighborhood 
around each state. 

For each state $i\in \widetilde S_\Prod$, We define the safety score $\mathtt{safe}(i)$ as the ratio of 
the neighboring states not in $\bar S^{\star}_{\Prod}$ over all neighboring states of $i$.   Recall that $\bar S^{\star}_{\Prod}$ is the set of states with $0$ probability of reaching the goal states $S^{\star}_{\Prod}$.  To be more specific, we define:
\begin{equation}
	\label{eq:safety}
	\mathtt{safe}(i) := \frac{\sum_{j
	\in N(i)}I(j)}{|N(i)|} ,
\end{equation}
where $I(i)$ is an indicator function such that $I(i)=1$ if and only if $i\in \widetilde S_\Prod\setminus \bar S^\star_\Prod$ and $I(i)=0$ if otherwise.
A higher safety score for the current state implies that it is less likely to reach $\bar S^{\star}_{\Prod}$ in the near future.  Furthermore, we define the progress score of a state $i\in \widetilde S_\Prod$ as 
$\mathtt{progress}(i):=\min_{j\in S^\star_\Prod}d(i, j)$, which is the minimum number 
of transitions from $i$ to any goal state.  

We can now present the definition of our RSP.  Let $\btheta:=[\theta_{1},\theta_{2}]^{\mathtt{T}}$.  We define:
\bea
\label{eq:rspeq}
	&&a \left( \btheta, i,e\right)\n &=& \mathtt{exp}\Big(\theta_1 \sum_{j\in N(i)} \mathtt{safe}(j) \widetilde P^{\N}_{\Prod} \left(i, e , j \right)\n &&  +\theta_2 \sum_{j\in N(i)}\left( \mathtt{progress} \left( j \right) - \mathtt{progress}\left(i \right) \right) \n&&\widetilde P^{\N}_{\Prod}  \left( i, e , j \right)\Big),
\eea
where $\mathtt{exp}$ is the exponential function.  Note that $a(\btheta, i, e)$ is the combination of the expected safety score of the next state applying the action sequence $e$, and the expected improved progress score from the current state applying $e$, weighted by $\theta_{1}$ and $\theta_{2}$.   We assign the probability of pick the action sequence $e$ at $i$ proportional to the combined score $a(\btheta,i,e)$.  Hence, the probability to pick action sequence $e$ at state $i$ is defined as:
\be
\label{eq:probaseq}
\tilde\mu_{\btheta} \left(i, e \right) = \frac{ a \left( \btheta, i, e\right)}{ \sum_{e\in E(i)} a \left( \btheta, i, e\right)}.
\ee

Note that, if the action sequence $e=u_{1}u_{2}\ldots u_{t}$ is picked, only the first action $u_{1}$ is applied.  Hence, at stat $i$, the probability that an action $u\in \widetilde A_{\Prod}(i)$ can be derived from Eq. \eqref{eq:probaseq}:
\begin{equation}
\mu_{\btheta} \left( i, u \right) = \sum_{\{e\in E(i) \st e=uu_{2}\ldots u_{t}\}} \tilde\mu_{\btheta}(i,e),
\end{equation}
which completes the definition of the RSP.

\subsection{Overall Algorithm}
\label{sec:sec:overallAlg}
We now connect all the pieces together and present the overall algorithm giving a solution to Prob. \ref{prob:mainprob}. 
\begin{algorithm}
\small
\caption{Overall algorithm providing a solution to Prob. \ref{prob:mainprob}}
\begin{algorithmic}[1]
\label{alg:overall}
\REQUIRE A labeled NTS $\mathcal M^{\N}=(Q, q_{0}, U, A, P^{\N}, \Pi, h)$ modeling a robot in a partitioned environment,  LTL formula $\phi$ over $\Pi$, and a simulator to compute $P(q,u,\cdot)$ given a state-action pair $(q,u)$
\STATE Translate the LTL formula $\phi$ to a DRA $\mathcal R_{\phi}$
\STATE Generate the product NTS $\Prod^{\N}=\M^{\N}\times \R_{\phi}$
\STATE Find the union of all AMECs $S^{\star}_{\Prod}$ associated with $\Prod^{\N}$
\STATE Convert from an MRP to an SSP and generate $\widetilde\Prod^{\N}$
\STATE Obtained the RSP $\mu_{\btheta}$ with $\Prod^{\N}$
\STATE Execute Alg. \ref{alg:A-C} with $\widetilde\Prod^{\N}$ and $\mu_{\btheta}$ as inputs until $||\nabla\bar{\alpha}(\btheta^{\star})||\leq \epsilon$ for a $\btheta^{\star}$ and a given $\epsilon$
\ENSURE RSP $\mu_{\theta}$ and $\theta^{\star}$ locally maximizing the probability of satisfying $\phi$ with respect to $\btheta$ up to a threshold $\epsilon$
\end{algorithmic}
\end{algorithm}
\begin{proposition}
 Alg. \ref{alg:overall} returns in finite time with $\btheta^{\star}$ locally maximizing the probability of the RSP $\mu_{\btheta}$ satisfying the LTL formula $\phi$.
\end{proposition}
\begin{proof}
 In \cite{estanjini2011least}, we have shown that the actor-critic algorithm used in this paper returns in finite time with a locally optimal $\btheta^{\star}$ such that $||\nabla\bar{\alpha}(\btheta^{\star})||\leq \epsilon$ for a given $\epsilon$.  We have shown throughout the paper that the optimal policy maximizing the probability of reaching $S^{\star}_{\Prod}$ on $\Prod$ is a policy maximizing the probability of satisfying $\phi$.   We also showed throughout the paper that the SSP problem, as well as the RSP $\mu_{\btheta}$ can be constructed without the transition probabilities, and only with $\M^{\N}$.  Therefore, Alg. \ref{alg:overall} produces an RSP maximizing the probability of satisfying $\phi$ with respect to $\btheta$ up to a threshold $\epsilon$.
\end{proof}

\section{Hardware-in-the-loop simulation}
\label{sec:Hardware}
We test the algorithms proposed in this paper through hardware-in-the-loop simulation for the RIDE environment (as shown in Fig. \ref{fig:robotandsmallenv}).  The transition probabilities are computed by an accurate simulator of RIDE as needed.
We apply both LTL control synthesis methods of linear programming (exact solution) and actor-critic (approximate solution) and compare the results.  

\subsection{Environment}

In this case study, we consider an environment whose topology is shown in Fig. \ref{fig:SmallEnv}.  This environment is made of square blocks forming 164 corridors and 84 intersections.  The corridors ($C_1, C_2, \dots, C_{164}$) shown as white regions in Fig. \ref{fig:SmallEnv} are of three different lengths, one-, two-, and three-unit lengths.  The three-unit corridors are used to build corners in the environment.  The intersections ($I_1, I_2,\ldots,I_{84}$) are of two types, three-way and four-way, and are shown as grey blocks in Fig. \ref{fig:SmallEnv}.  The black regions in this figure represent the walls of the environment.  Note that there is always a corridor between two intersections.  

\begin{figure}
   \center
\includegraphics[width=82mm]{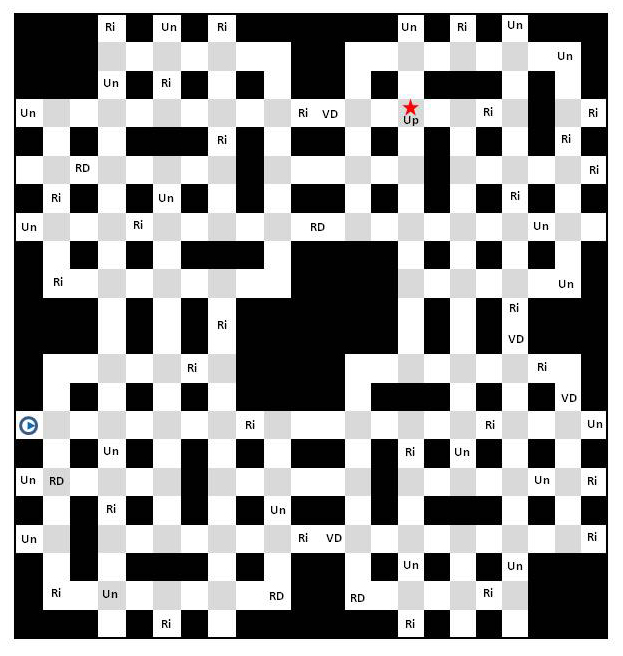}
\caption{Schematic representation of the environment with 84 intersections and 164 corridors.  The black blocks represent walls, and the grey and white regions are intersection and corridors, respectively.  There are five properties of interest in the regions indicated with \textbf{VD} = \emph{ValuableData}, \textbf{RD} = \emph{RegularData}, \textbf{Up} = \emph{Upload}, \textbf{Ri} = \emph{Risky}, and \textbf{Un} = \emph{Unsafe}.  The initial position of the robot is shown with a blue disk and the upload region is indicated with a red star.}\label{fig:SmallEnv}
\end{figure}


There are five properties of interest (observations) associated with the regions of the environment.  These properties are: \textbf{VD} = \emph{ValuableData} (regions containing valuable data to be collected), \textbf{RD} = \emph{RegularData} (regions containing regular data to be collected), \textbf{Up} = \emph{Upload} (regions where data can be uploaded), \textbf{Ri} = \emph{Risky} (regions that could pose a threat to the robot), and \textbf{Un} = \emph{Unsafe} (regions that are unsafe for the robot).

\subsection{Construction of the MDP model}
\label{sec:cons_MDP}

The robot is equipped with a set of feedback control primitives (actions) - $\mathtt{FollowRoad}$, $\mathtt{GoRight}$, $\mathtt{GoLeft}$, and $\mathtt{GoStraight}$.  The controller $\mathtt{FollowRoad}$ is only available (enabled) at the corridors. At four-way intersections, controllers are $\mathtt{GoRight}$, $\mathtt{GoLeft}$, and $\mathtt{GoStraight}$.   At three-way intersections, depending on the shape of the intersection, two of the four controllers are available.  Due to the presence of noise in the actuators and sensors, however, the resulting motion may be different than intended.  Thus, the outcome of each control primitive is characterized probabilistically.  

To create an MDP model of the robot in RIDE, we define each state of the MDP as a collection of two adjacent regions (a corridor and an intersection).  For instance the pairs $C_1$-$I_2$ and $I_3$-$C_4$ are two states of the MDP.  Through this pairing of regions, it was shown that the Markov property (\ie the result of an action at a state depends only on the current state) can be achieved \cite{LaWaAnBe-ICRA10}. The resulting MDP has 608 states.

The set of actions available at a state is the set of controllers available at the last region corresponding to the state. For example, when in state $C_1$-$I_2$ only those actions from region $I_2$ are allowed.  
Each state of the MDP whose second region satisfies an observation in $\Pi$ is mapped to that observation.  

To obtain transition probabilities, we use an accurate simulator (see Fig. \ref{fig:snap_sim}) incorporating the motion and sensing of an iRobot Create platform with a Hokoyu URG-04LX laser range finder, APSX RW-210 RFID reader, and an MSI Wind U100-420US netbook (the robot is shown in Fig. \ref{fig:robotandsmallenv}) in RIDE.  Specifically, it emulates experimentally measured response times, sensing and control errors, and noise levels and distributions in the laser scanner readings.  More detail for the software implementation of the simulator can be found in \cite{LaWaAnBe-ICRA10}.
We perform a total of 1,000 simulations for each action available in each MDP state. 

\begin{figure}
   \center
   \includegraphics[width=65mm]{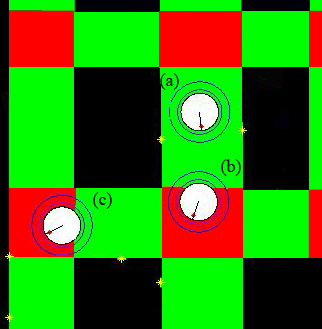} \\
   \caption{Simulation snapshots. The white disk represents the robot and the different circles around it indicate different "zones" in which different controllers are activated. The yellow dots represent the laser readings used to define the target angle. (a) The robot centers itself on a stretch of corridor by using $\mathtt{FollowRoad}$; (b) The robot applies $\mathtt{GoRight}$ in an intersection; (c) The robot applies $\mathtt{GoLeft}$.}
   \label{fig:snap_sim}
\end{figure}

\subsection{Task specification and results}
We consider the following mission task:

\noindent {\bf Specification:} Reach a location with \emph{ValuableData} (\textbf{VD}) or \emph{RegularData} (\textbf{RD}), and then reach \emph{Upload} (\textbf{Up}).  Do not reach \emph{Risky} (\textbf{Ri}) regions unless eventually reach a location with \emph{ValuableData} (\textbf{VD}).  Always avoid \emph{Unsafe} (\textbf{Un}) regions until \emph{Upload} (\textbf{Up}) is reached (and mission completed). 

The above task specification can be translated to the LTL formula:
\bea
\label{eq:LTLformulaex}
 \phi&:=&\Event\textbf{Up} \Land (\neg \textbf{Un} \Until \textbf{Up})\Land \Always(\textbf{Ri}\longrightarrow \Event \textbf{VD})  \n
&&\Land \Always (\textbf{VD} \vee \textbf{RD}\longrightarrow \Next \Event \textbf{Up})
\eea

\begin{figure}[t]
\begin{center}
\includegraphics[scale=.6]{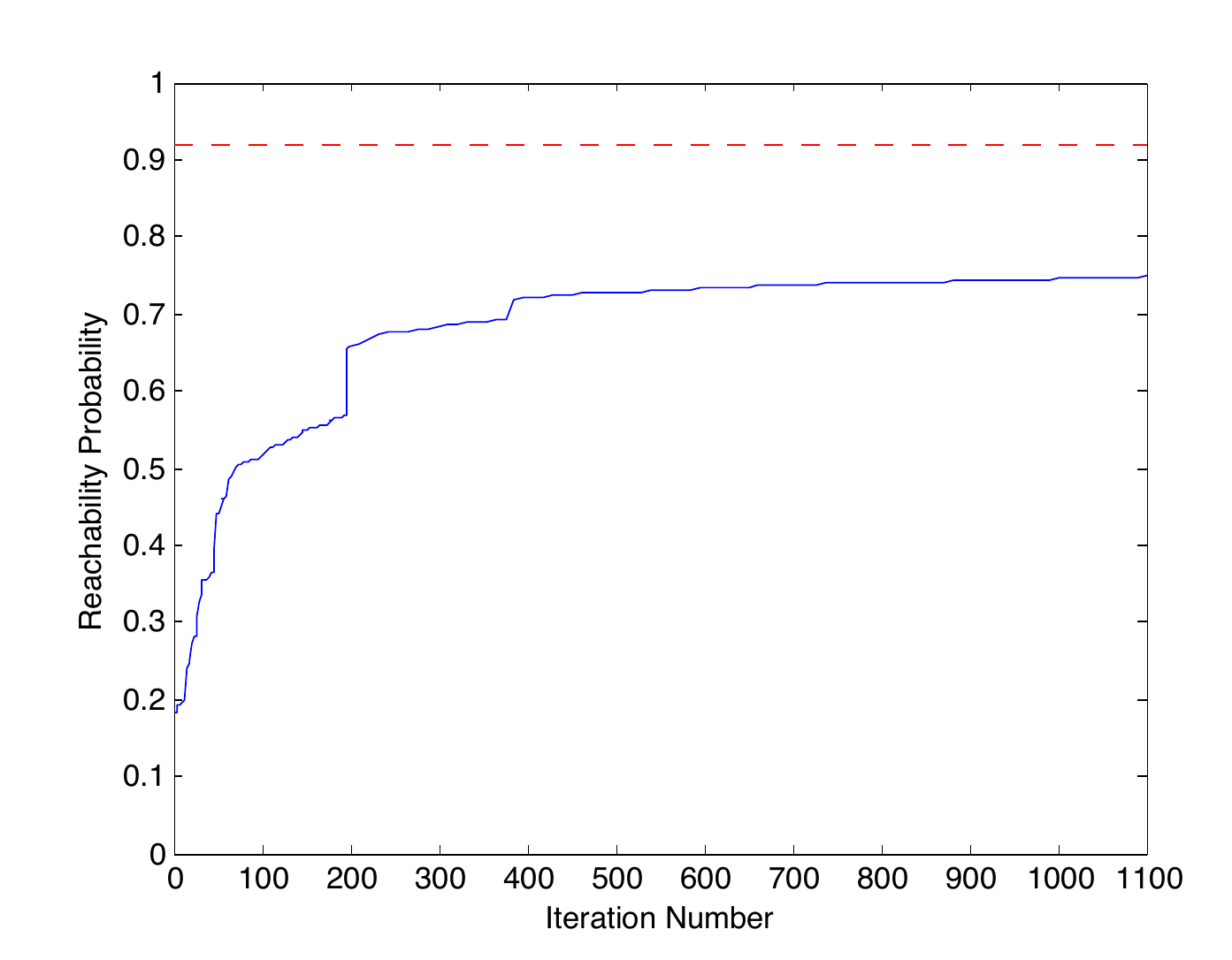}
\end{center}
\caption{The optimal solution (the maximal probability of satisfying the specification) is shown with the dashed line,  and the solid line represents the exact reachability probability for the RSP as a function of the number of iterations applying the proposed algorithm.}
 \label{small2}
\end{figure}

The initial position of the robot is shown as a blue circle in Fig. \ref{fig:SmallEnv} with the orientation towards the neighboring intersection.  We used the computational frameworks described in this paper to find the control strategy maximizing the probabilities of satisfying the specification. The size of the DRA is 17 which results in the product MDP with $10336$ states.  By applying both methods of linear programming (exact solution) and actor-critic (approximate solution), we found the maximum probabilities of satisfying the specification were $92\%$ and $75\%$, respectively.  
The graph of the convergence of the actor-critic solution is shown in Fig. \ref{small2}.  The parameters for this examples are: $\lambda=0.9$,  and the initial $\btheta=[5, -0.5]^{\mathtt{T}}$.  The look-ahead window $t$ for the RSP is $2$.  

It should be emphasized that, we only compute the transition probabilities along the sample path.  Thus, when Alg. \ref{alg:overall} is completed (at iteration $1100$), at most $1100$ transition probabilities of state-action pairs were computed.  In comparison, in order to solve the probability exactly, arround $30000$ transition probabilities of state-action pairs must be computed.

\section{Conclusions}
\label{sec:conclusions}
We presented a framework that brings together an approximate dynamic programming computational method of the actor critic type, with formal control synthesis for Markov Decision Processes (MDPs) from temporal logic specifications.  We show that this approach is particular suitable for problems where the transition probabilities of the MDP are difficult or computationally expensive to compute, such as for many robotic applications.  We show that this approach effectively finds an approximate optimal policy within a class of randomized stationary polices maximizing the probability of satisfying the temporal logic formula.   Future direction includes extending this result to multi-robot teams, examining exactly how to choose an appropriate look-ahead window when designing the RSP, and applying the result to more realistic problem settings with the MDP containing possibility millions of states.


\bibliographystyle{IEEEtran}

\end{document}

%% file: lpsymbols.tex
\def\ba{{\mathbf a}}
\def\bb{{\mathbf b}}

\def\bh{{\mathbf h}}

\def\bq{{\mathbf q}}
\def\br{{\mathbf r}}

\def\bx{{\mathbf x}}  

\def\bz{{\mathbf z}}
\def\bA{{\mathbf A}}

\def\texitem#1{\par\smallskip\noindent\hangindent 25pt
               \hbox to 25pt {\hss #1 ~}\ignorespaces}

\newcommand{\btheta}{\boldsymbol{\theta}}

\newcommand{\bpsi}{\boldsymbol{\psi}}